\documentclass[letterpaper]{article} 
\usepackage{aaai2027}  
\nocopyright
\usepackage[hyphens]{url}  
\usepackage{graphicx} 
\usepackage{natbib}  
\usepackage{caption} 
\usepackage{algorithm}
\usepackage{algorithmic}

\usepackage{amsmath}
\usepackage{amssymb}
\usepackage{amsthm}
\newtheorem{theorem}{Theorem}
\newtheorem{lemma}{Lemma}

\newtheorem{proposition}{Proposition}
\newtheorem{remark}{Remark}
\newtheorem{assumption}{Assumption}

\usepackage{booktabs}
\usepackage{multirow}

\newcommand{\xtil}{\tilde{\mathbf{x}}}

\newcommand{\CVF}{\mathrm{CV}_F}

\usepackage{newfloat}
\usepackage{listings}
\DeclareCaptionStyle{ruled}{labelfont=normalfont,labelsep=colon,strut=off} 
\floatstyle{ruled}
\newfloat{listing}{tb}{lst}{}
\floatname{listing}{Listing}

\usepackage{booktabs}

\title{
FedFIbOS: Fisher Importance based Optimal Submodelling for Heterogeneous Federated Learning
}

\author{
Yasmeen Afzal\textsuperscript{\rm 1},
Jeremiah D. Deng\textsuperscript{\rm 1},
Haibo Zhang\textsuperscript{\rm 2}
}

\affiliations{
\textsuperscript{\rm 1}School of Computing, University of Otago, New Zealand\\
\textsuperscript{\rm 2}School of Computer Science and Engineering, University of New South Wales, Australia\\
afzya097@student.otago.ac.nz, jeremiah.deng@otago.ac.nz, haibo.zhang@unsw.edu.au
}

\begin{document}

\maketitle

\begin{abstract}
Heterogeneous federated learning requires clients with diverse
computational capacities to collaboratively train a global model,
where each client trains a capacity-constrained submodel.
Existing methods select submodel parameters using heuristic importance
measures---most prominently parameter magnitude---without theoretical
justification for why these measures support convergence.
We identify a fundamental gap: existing parameter selection criteria
lack theoretical grounding in the convergence framework, partial client participation introduces additional estimation effects in the Fisher scores.
We propose \textbf{FedFIbOS}: Fisher Importance-based Optimal Submodelling for heterogeneous federated learning, using Fisher Information in a principled criterion derived from minimizing submodel masking error.
We theoretically formulate submodel selection through a
Fisher-weighted quadratic masking surrogate and show that the
raw Fisher top-$k$ rule implemented by FedFIbOS solves this
surrogate under a Fisher-dominant ranking condition. The resulting method retains the convergence structure of the underlying masked federated optimization bound.
Fisher scores are efficiently estimated from empirical diagonal Fisher information using squared gradients, enabling stable and adaptive parameter selection without additional optimization overhead.
Experiments on CIFAR-10, CIFAR-100, and AGNews under pathological and
Dirichlet non-IID settings show FedFIbOS achieves ${\approx}10\%$ higher
accuracy than the state of the art, with improvements becoming more pronounced under stronger
heterogeneity.
\end{abstract}
\section{Introduction}
\label{sec:intro}
 
Federated learning (FL) enables
collaborative model training across distributed clients without sharing
private data \cite{mcmahan2017communication}. A central challenge in practical deployments is
\emph{model heterogeneity}: clients span diverse hardware with vastly
different compute, memory, and communication
budgets~\citep{fang2025automated,pfeiffer2023federated, chen2023efficient}. Model-heterogeneous
federated learning (MHFL) addresses this by assigning each client a submodel derived from the global model and tailored to its resource constraints. The quality of the extracted submodels can directly affect
convergence speed and final accuracy \cite{liao2023adaptive, yi2024fedp3}.

Existing MHFL methods primarily construct client submodels using three strategies: static slicing~\cite{diao2021heterofl}, dynamic rolling-based extraction~\cite{alam2022fedrolex,liao2025fedbrb}, and importance-based selection\cite{wu2024fiarse,chen2023efficient,liao2023adaptive, yi2024fedp3}. Static and dynamic strategies determine submodels according to predefined architectural rules rather than optimization relevance, which can lead to suboptimal parameter allocation, larger gradient mismatch, and increased client drift under heterogeneous data distributions.   Importance-based methods address this partially, but
existing approaches rank parameters by magnitude, a
criterion independent of each client's local data
distribution, thereby providing no theoretical guarantee of
minimising the masking approximation error that
directly governs convergence~\citep{wu2024fiarse}.

To address these limitations, we propose
\textbf{FedFIbOS} (\textbf{F}isher-\textbf{I}mportance-\textbf{b}ased
\textbf{O}ptimal \textbf{S}ubmodel Extraction for Heterogeneous
Federated Learning), which constructs client-specific submodels by selecting parameters based on their importance estimated using Fisher information.
Instead of selecting parameters based on
their position within the network, FedFIbOS retains the
most optimization-relevant parameters for each client
while respecting its computational capacity. As illustrated in
Figure.~\ref{fig:FIbASE}, the colored connections represent the relative Fisher importance of the selected parameters. Red edges correspond to the highest-ranked parameters, which have the greatest estimated impact on the loss and are therefore prioritised for inclusion in the submodel. Yellow edges indicate parameters of moderate importance, which are included when model capacity is available. Green edges represent lower-ranked parameters that are assigned only to high-capacity clients with sufficient computational resources.
Consequently, low-capacity clients receive only the most important (red) parameters, medium-capacity clients receive both red and yellow parameters, while high-capacity clients receive parameters in all three groups.

\paragraph{Main Contributions}
Our main contributions are summarized as follows.

\begin{itemize}

\item \textbf{Fisher Guided submodel extraction.}
We propose a Fisher Importance based Optimal Submodelling method that employs empirical diagonal Fisher information to identify the most informative parameters for submodel extraction in heterogeneous federated learning.

\item \textbf{Theoretical foundation for Fisher-guided selection.}
We formulate the Submodel Selection Problem (SSP), identify the conditions under which different importance
criteria are equivalent, establish the optimality of raw Fisher selection under the SSP formulation under the Fisher-dominant ranking condition.

\item \textbf{Convergence analysis.}
We prove that Fisher-guided selection preserves the standard
$\mathcal{O}(1/\sqrt{T})$ convergence rate, where T is the number of communication rounds while yielding a no-larger selection-dependent contribution to
the convergence bound than magnitude-based selection, with strict
improvement when the corresponding masking-error coefficients
differ.

\item \textbf{Extensive experimental validation.}
Experiments on CIFAR-10 , CIFAR-100 , and AGnews under diverse model and data heterogeneity settings demonstrate consistent improvements over state-of-the-art MHFL methods and validate the theoretical analysis.

\end{itemize}

\begin{figure}[t]
\centering
\includegraphics[width=\linewidth]{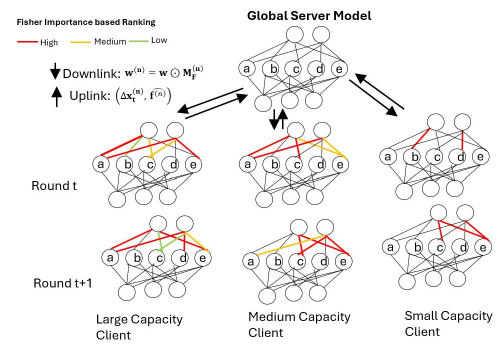}
\caption{Overview of Submodel Extraction for model training in FedFIbOS.}
\label{fig:FIbASE}
\end{figure}

\section{Prior work and its limitation.}
\paragraph{Static submodel extraction.}
 Early heterogeneous-FL methods such as HeteroFL~\citep{diao2021heterofl} and related approaches~\citep{horvath2021fjord, li2021fedmask, li2020lotteryfl} assign clients subnetworks or masks with predefined or client-specific structures, reducing computation and communication but limiting adaptation of the extracted submodels to evolving training dynamics.

\paragraph{Dynamic submodel extraction.}
FedRolex~\citep{alam2022fedrolex} rolls the submodel selection
channel-wise across rounds, ensuring equal training opportunity
for each channel.
However, on square-channel tensors the rolling reduces to diagonal
cycling, leaving large parameter regions permanently untrained
regardless of the number of rounds.
FedBRB~\citep{liao2025fedbrb} addresses this by partitioning the global
model into blocks and rolling block-wise, with weighted broadcast
to accelerate information sharing across submodels, achieving full
parameter coverage that FedRolex cannot guarantee.
Despite this improvement, both methods remain data-agnostic: neither
asks which parameters matter most for the loss under each client's
local data distribution.
A capacity-constrained client trains whichever slice the rolling
schedule assigns, not the slice most important for convergence.
\cite{wen2022federated,liao2023adaptive}
\paragraph{Importance-aware submodel extraction.}

FIARSE~\citep{wu2024fiarse} advances importance-aware submodel
extraction by selecting parameters based on magnitude and establishing
the first convergence guarantee for model-heterogeneous FL under
partial participation.
However, its convergence analysis treats the masking error as a bounded black-box term, without analysing its structure or the optimality of the underlying selection criterion.
FedLAGC~\citep{hu2026fedlagc} addresses a related but distinct
limitation by introducing layer-adaptive importance guided by
gradient norms, yet its selection criterion remains a heuristic
with no connection to the quantity that governs convergence.
This connection between the selection criterion and the quantity
governing convergence has remained unestablished across all existing
approaches, leaving submodel selection without a principled
theoretical foundation.
\paragraph{Fisher Information for model compression.}
Classical pruning methods use curvature to quantify parameter importance. Optimal Brain Damage~\citep{lecun1989optimal} uses a diagonal Hessian approximation, while Optimal Brain Surgeon~\citep{hassibi1992second} accounts for parameter interactions through the inverse Hessian. Under standard regularity conditions, Fisher Information provides a positive-semidefinite, gradient-based surrogate for the Hessian, motivating its use for efficient importance estimation~\citep{molchanov2019importance}. Recent work extends Fisher-based importance to model sparsification, including FisherLAS~\citep{sun2026fisher} for LLMs. In federated learning, Fed-HeLLo~\citep{zhang2025fed} uses layer-wise Fisher information for heterogeneous LoRA allocation, while other works explore Fisher-guided pruning, parameter selection, and privacy-aware optimization~\citep{chen2024sparsified,liu2024fisher}. FedFIbOS instead uses Fisher information for submodel selection under non-IID heterogeneity and partial participation, with formal convergence guarantees.


\section{Problem Formulation}
\label{sec:prob}
Consider a FL system with $N$ clients, which each local dataset for client $n$ is drawn from
distributions $\{P_n\}_{n=1}^N$.
At each round $t$, a subset $\mathcal{A}\subseteq[N]$
of $A$ clients is sampled to participate  training.
Each client $n$ has a compute capacity $\gamma_n\in(0,1]$
and is assigned a submodel containing $k_n=\lceil\gamma_n d\rceil$
parameters from the global model $\xtil\in\mathbb{R}^d$.
A binary mask $M^{(n)}\in\{0,1\}^d$,
where $\|M^{(n)}\|_0=k_n$, specifies the parameters included in the submodel.
The global learning objective is to minimise
\begin{equation}
  \mathcal{F}(\xtil) =
  \frac{1}{N}\sum_{n=1}^N
  \mathbb{E}_{(\mathbf{x},y)\sim P_n}[\mathcal{L}(\xtil;\mathbf{x},y)].
  \label{eq:global_obj}
\end{equation}

When client $n$ trains on $\xtil_t\odot M^{(n)}$ rather
than the full model, local gradient computation incurs a
\emph{masking approximation error}:
\[
  e_t^{(n)} \;\triangleq\;
  \bigl\|\nabla F_n(\xtil_t)\odot M^{(n)}
  - \nabla_{\xtil_t}F_n(\xtil_t\odot M^{(n)})
  \bigr\|_2^2.
\]
Following~\citep{wu2024fiarse}, we bound this error as:
\begin{equation}
  e_t^{(n)}
  \;\leq\; \delta_t^2\,\|\xtil_t\|_2^2,
  \label{eq:mask_ine}
\end{equation}
where $\delta_t^2\geq 0$ is the smallest constant for
which Inequality (\ref{eq:mask_ine}) holds.
The parameter subset
$M^{(n)}$ directly determines
$e_t^{(n)}$. Fisher-guided selection includes the
parameters whose exclusion would induce the largest
gradient discrepancy, thereby minimising
$e_t^{(n)}$ and tightening the
convergence bound.
The selection criterion determines $\delta_t^2$: a smaller
$\delta_t^2$ means the bound is tighter and local
optimisation is more faithful to the full model.

This error directly governs convergence.
The bound for masked heterogeneous FL~\citep{wu2024fiarse}
takes the form:
\begin{equation}
  \min_{t\in[T]}\|\nabla F(\xtil_t)\|_2^2
  \leq
  \underbrace{C_1(T)}_{\to\,0}
  +
  \underbrace{C_2}_{\text{heterogeneity}}
  +
  \underbrace{
    \frac{32N}{T}\sum_t\delta_t^2\|\xtil_t\|_2^2
  }_{\text{masking error (Term 3)}},
  \label{eq:conv_form}
\end{equation}
where $C_1(T)\to 0$ and $C_2$ depends only on data
heterogeneity.
Term~3 is the only term governed by the choice of $S$:
minimising $\delta_t^2$ directly tightens the convergence neighbourhood.

To obtain a tractable selection criterion, let
$\Delta_t(S)=\tilde{x}_t\odot M_S-\tilde{x}_t$.
A first-order expansion of the local gradient gives
\[
\nabla F_n(\tilde{x}_t\odot M_S)
\approx
\nabla F_n(\tilde{x}_t)
+
H_n(\tilde{x}_t)\Delta_t(S).
\]
Near a stationary point, the residual term
$\nabla F_n(\tilde{x}_t)\odot(1-M_S)$ is negligible, so the
selection-dependent masking discrepancy is locally approximated by
\[
\|H_n(\tilde{x}_t)\Delta_t(S)\|_2^2.
\]
Under a diagonal-Hessian approximation and the local
Fisher--Hessian correspondence, this yields the quadratic surrogate
\begin{equation}
\epsilon_t^{(n)}(S)
\triangleq
\sum_{j\notin S}
\left(F^{(n)}_{t,j}\right)^2
\tilde{x}_{t,j}^2.
  \label{eq:eps}
\end{equation}

This surrogate provides a tractable approximation to the
selection-dependent masking error.

The resulting Submodel Selection Problem (SSP) is:
\begin{equation}
S^{\star(n,t)}
=
\arg\min_{S\subseteq[d],\,|S|=k_n}
\epsilon_t^{(n)}(S).
  \label{eq:SSP}
\end{equation}
Its exact solution is Fisher-weighted top-$k$ selection:
\begin{equation}
S^{\star(n,t)}
=
\operatorname{TopK}_{k_n}
\left(
\left(F_{t,j}^{(n)}\right)^2\tilde{x}_{t,j}^2
\right).
\end{equation}
Existing magnitude-based methods~\citep{wu2024fiarse}
instead minimise $\sum_{j\notin S}\tilde{x}_{t,j}^2$,
which omits the Fisher weights $(F_j^{(n)})^2$ and
provides no guarantee of solving~\eqref{eq:SSP}.
This surrogate is particularly unreliable under non-IID
data: $|\xtil_j|$ is a property of the shared global
model identical across all clients, whereas $F_j^{(n)}$
captures the sensitivity of $F_n$ to parameter $j$ under
client $n$'s local distribution $P_n$.
Since $F_{t,j}^{(n)}$ contributes directly to the Fisher-weighted
importance score in (4), it provides the basis for the raw-Fisher
criterion used by FedFIbOS; under Assumption 4, this criterion
induces the same ordering as the exact SSP objective.

We further use $\CVF$ to quantify heterogeneity in Fisher importance scores and empirically examine its relationship with selection divergence, where
$\CVF=\mathrm{std}(F_j^{(n)})/\mathrm{mean}(F_j^{(n)})$.

\section{FedFIbOS}

We propose FedFIbOS, a principled
importance-aware submodel extraction framework for
model-heterogeneous federated learning. 
Although it retains the optimization framework of FIARSE, 
what makes the difference is that FedFIbOS measures parameter importance using
Fisher Information, which directly quantifies the sensitivity
of the local objective to parameter perturbations.
\paragraph{Fisher Information.}

The diagonal Fisher Information score for parameter $j$ at client $n$:
\begin{equation}
  F_j^{(n)} =
  \mathbb{E}_{(x,y)\sim P_n}\!\!\left[
  \left(\frac{\partial\log p(y|x,\theta)}{\partial\theta_j}\right)^{\!2}
  \right]
  = \mathbb{E}_{P_n}\!\left[
  \left(\frac{\partial\mathcal{L}}{\partial\theta_j}\right)^{\!2}
  \right].
  \label{eq:fisher_def}
\end{equation}
We use $F_{t,j}^{(n)}$ to denote the Fisher score evaluated at
the round-$t$ model $\tilde{x}_t$; when the round index is
unambiguous, we write $F_j^{(n)}$ for brevity.
\begin{equation}
  \hat{f}_j^{(n)} =
  \frac{1}{KB}\sum_{k=0}^{K-1}\sum_{b=1}^{B}
  \left(\frac{\partial\mathcal{L}(x_{t,k}^{(n)},\xi_b)}
  {\partial\theta_j}\right)^{\!2}.
  \label{eq:fisher_est}
\end{equation}
The empirical diagonal Fisher Information is computed from the squared
gradients already available during local optimization, thereby avoiding
any additional gradient evaluations. To reduce the variance of the
estimated importance scores, the server maintains an exponential moving
average (EMA) of the aggregated Fisher estimates over communication
rounds.
The server maintains an EMA across rounds:
\begin{equation}
  \bar{f}_j^{(n,t+1)} =
  \alpha\,\bar{f}_j^{(n,t)} + (1{-}\alpha)\,\hat{f}_j^{(n)}.
  \label{eq:ema}
\end{equation}
 
\subsection{Fisher Importance based Submodel Extraction}
 
The quadratic SSP is minimized by Fisher-weighted scores
$I_{t,j}^{(n)}=(F_{t,j}^{(n)})^2\tilde{x}_{t,j}^2$. FedFIbOS instead
implements raw Fisher top-$k$ selection using the maintained score
\begin{equation}
  M_F^{(n)}(\tilde{x}) =
\left\{
j :
\bar{f}^{(n,t)}_j
\in
\operatorname{TopK}_{\lceil\gamma_n d\rceil}
(\bar{f}^{(n,t)})
\right\}.
  \label{eq:fgiase_mask}
\end{equation}

For each participating client, FedFIbOS constructs a client-specific submodel using the maintained Fisher importance estimates. Given the target model ratio $\gamma_n$, the server identifies the $\lceil\gamma_n d\rceil$ most important parameters and generates a binary mask that determines the subset of the global model transmitted for local training. Under Assumption 4, introduced in the theoretical analysis, the
Fisher and Fisher-weighted rankings coincide; hence this implemented
raw-Fisher rule solves the quadratic SSP under that condition.

\begin{algorithm}[t]
\caption{\textbf{FedFIbOS:}
  Fisher-Importance-based Optimal Submodel
Extraction }
\label{alg:fgiase}
\begin{algorithmic}[1]

\STATE \textbf{Input:} $\tilde{x}_0\in\mathbb{R}^d$,
  $\{\gamma_n\}_{n=1}^N$,
  rounds $T$, local steps $K$, batch size $B$,
  rates $\eta_l, \eta_s$, EMA rate $\alpha\in(0,1)$

\STATE \textbf{Server Initialization:}
  $\bar{f}^{(n,0)} \leftarrow |\xtil_0|$,
  $\;\forall n\in[N]$

\FOR{$t = 0, 1, \ldots, T{-}1$}

  \STATE Sample $\mathcal{A}\subseteq[N]$,
    $|\mathcal{A}|{=}A$ uniformly without replacement

  \FOR{each $n\in\mathcal{A}$}
    \STATE $M_F^{(n)} \leftarrow
      \bigl\{j\in[d] : \bar{f}_j^{(n,t)} \in
      \mathrm{TopK}_{\gamma_n}\!(\bar{f}^{(n,t)})\bigr\}$,
     
    \STATE Send $\xtil_t \odot M_F^{(n)}$ to client $n$
  \ENDFOR

  \FOR{each $n\in\mathcal{A}$ in parallel}
    \STATE $\mathbf{x}_{t,0}^{(n)} \leftarrow \xtil_t \odot M_F^{(n)}$
    \FOR{$k = 0, \ldots, K{-}1$}
      \STATE $\mathbf{g}_{t,k+1}^{(n)} \leftarrow
        \nabla_{\mathbf{x}_{t,k}^{(n)}}
        F_n\!\left(\mathbf{x}_{t,k}^{(n)} \odot
        M_F^{(n)}\!\left(\mathbf{x}_{t,k}^{(n)}\right)\right)$
      \STATE $\mathbf{x}_{t,k+1}^{(n)} \leftarrow
        \mathbf{x}_{t,k}^{(n)} - \eta_l \cdot \mathbf{g}_{t,k+1}^{(n)}$
      \STATE $\hat{f}_j^{(n)} \mathrel{+}=
  \frac{1}{KB}
  \left(\mathbf{g}_{t,k+1,j}^{(n)}\right)^{2}$
    \ENDFOR
    \STATE $\Delta x_t^{(n)} \leftarrow
      \xtil_t - \mathbf{x}_{t,K}^{(n)}$
    \STATE Send $\Delta \mathbf{x}_t^{(n)}$ and
      $\hat{f}^{(n)}$ to server
  \ENDFOR

  \STATE $\xtil_{t+1} \leftarrow \xtil_t -
  \frac{\eta_s}{A}
  \sum_{n\in\mathcal{A}} \Delta \mathbf{x}_t^{(n)}$
  \FOR{each $n\in\mathcal{A}$}
    \STATE $\bar{f}^{(n,t+1)} \leftarrow
      \alpha\,\bar{f}^{(n,t)} +
      (1{-}\alpha)\,\hat{f}^{(n)}$
  \ENDFOR
  \STATE $\bar{f}^{(n,t+1)} \leftarrow \bar{f}^{(n,t)}$,
    $\quad\forall n\notin\mathcal{A}$

\ENDFOR

\end{algorithmic}
\end{algorithm}

\paragraph{Algorithm}
At each round, the server samples a set of participating clients and construct a client specific Fisher-guided mask
$M_F^{(n)}$ from the EMA $\bar{f}^{(n,t)}$
(Algorithm~\ref{alg:fgiase}, line~3-6). The mask retains the
$\lceil\gamma_n d\rceil$ parameters with the largest
estimated Fisher importance, and the corresponding
masked global submodel is transmitted to client $n$.
Each client performs $K$ local SGD steps on its masked
submodel while simultaneously accumulating squared
gradients to estimate the empirical Fisher information
(line~14).  The server
aggregates the received updates to form the next global
model (line~19), updates the EMA Fisher scores of
participating clients (lines~21--23), and keeps the EMA
scores of non-participating clients unchanged
(line~23). Proposition~\ref{prop:ema} characterizes the decay of initialization bias in the
EMA Fisher estimates under partial participation.

\section{Experiments}
\label{sec:Exp}

\paragraph{Datasets and Models.}
We evaluate the proposed FedFIbOS on three benchmark datasets spanning both computer vision and natural language processing tasks: CIFAR-10, CIFAR-100~\cite{krizhevsky2009learning}, and AGNews~\cite{zhang2015character}. For the computer vision datasets, we adopt ResNet-18 as the backbone model, while for AGNews we fine-tune a pretrained RoBERTa-Base model. This selection enables evaluation across diverse modalities and model architectures.

\paragraph{Data Heterogeneity.}
Experiments are conducted under two widely adopted non-IID partitioning strategies: pathological partitioning and Dirichlet partitioning. For pathological partitioning, we consider two levels of heterogeneity. \emph{Low heterogeneity} assigns 5 and 50 classes per client for CIFAR-10 and CIFAR-100, respectively, whereas \emph{high heterogeneity} assigns only 2 and 20 classes per client. Since AGNews contains only four classes, pathological partitioning is not sufficiently representative; therefore, AGNews is evaluated only under Dirichlet partitioning. For Dirichlet partitioning, the concentration parameter is set to $\alpha=0.3$ for CIFAR-10 and CIFAR-100~\cite{krizhevsky2009learning}, and $\alpha=1.0$ for AGNews~\cite{zhang2015character}.

\paragraph{Model Heterogeneity.}
Following the model-heterogeneous federated learning setting, each client is assigned one of five model capacities corresponding to model ratios ${1.0, 0.5, 0.25, 0.125, 0.0625}$, where a ratio of $1.0$ denotes the complete global model and smaller ratios correspond to progressively reduced submodels. We evaluate two client-capacity distributions. The \emph{uniform} distribution assigns 20\% of clients to each model size (a20-b20-c20-d20-e20). The \emph{non-uniform} distribution (a10-b10-c30-d30-e20) better reflects practical deployments by assigning 80\% of clients to resource-constrained devices (model ratios 0.25--0.0625) and the remaining 20\% to higher-capacity devices (model ratios 0.5 and 1.0). For CIFAR-10 and CIFAR-100, experiments are conducted with 100 clients and a client participation rate of 10\% per round, and training for $800$ communication
rounds. For AGNews, we use 200 clients with the same participation rate of 10\%, and training for $300$ communication
rounds.

\begin{figure*}[t]
    \centering
    \includegraphics[width=.75\textwidth]{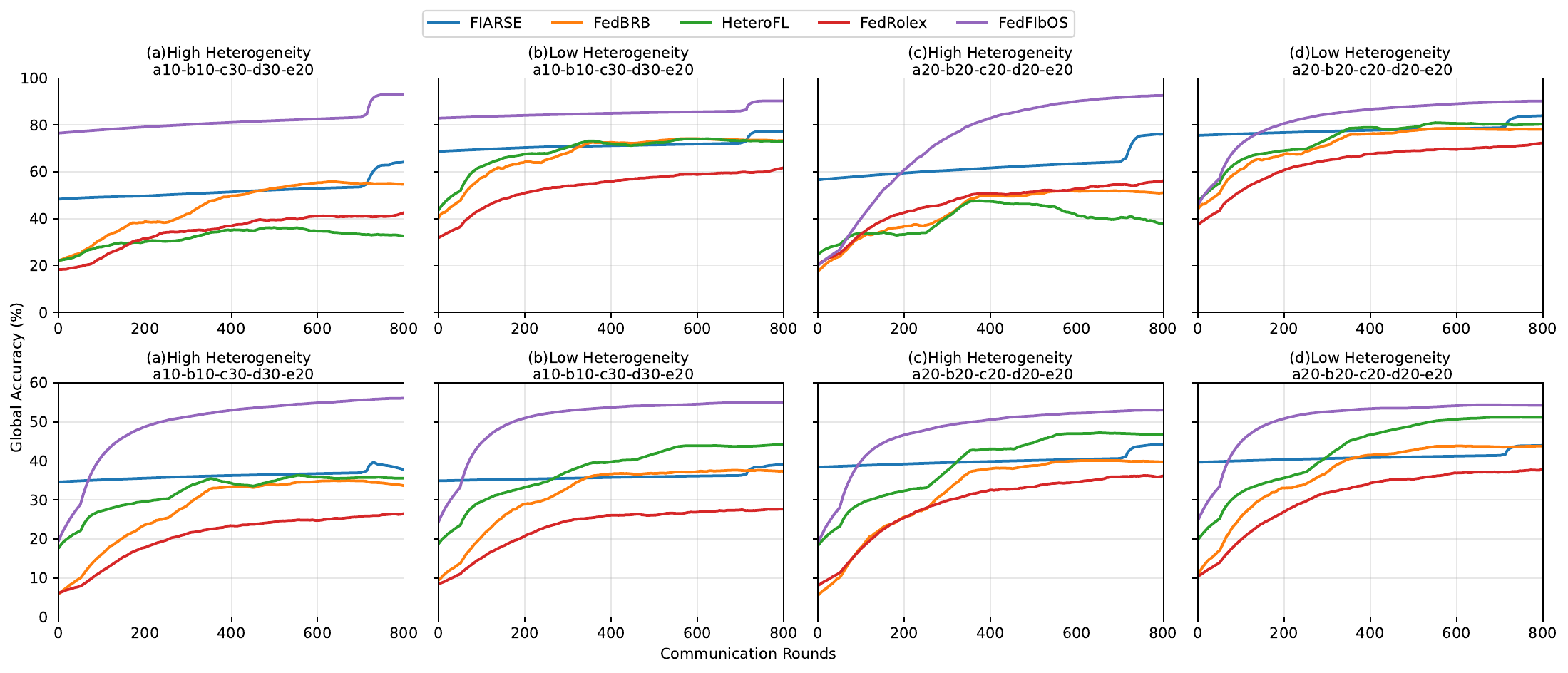}
    \caption{Global test accuracy comparison of FIARSE, FedBRB, HeteroFL, FedRolex, and FedFIbOS on CIFAR-10(first row) and CIFAR-100(second row) under four heterogeneous settings. Curves show the moving average of global test accuracy over communication rounds.}
    \label{fig:global_accuracy}
\end{figure*}
\paragraph{Structural divergence and resource sensitivity.}
Figure~\ref{fig:jaccard} shows a consistent two-phase
behaviour across all capacity-constrained model sizes.
In early rounds, Fisher and magnitude selections largely
agree because the Fisher EMA is warm-started from
parameter magnitude (Algorithm~\ref{alg:fgiase},
line~2), before gradient information is available.
As local gradient estimates accumulate and the EMA
converges toward the true Fisher distribution, the
two selected parameter sets diverge rapidly and
stabilise at a persistently low Jaccard plateau.
The observed persistent low plateau is an empirical finding; Lemma \ref{lem:equiv} provides a theoretical condition under which such selection divergence can occur.

The divergence intensifies under tighter capacity
constraints - Jaccard stabilises at ${\approx}0.42$
at $\gamma{=}0.5$ and falls to ${\approx}0.22$ at
$\gamma{=}0.0625$ - because a sharper selection
boundary amplifies the consequence of choosing the
wrong importance criterion.

We use Fisher Coefficient of Variation ($\CVF$) to measure Fisher score divergence, where
$\CVF=\mathrm{std}(F_j^{(n)})/\mathrm{mean}(F_j^{(n)})$. 
$\CVF=0$ iff all Fisher scores are equal;
otherwise $\CVF>0$.
Figure~\ref{fig:cvf} shows substantially greater relative dispersion in Fisher scores than in parameter magnitudes after the warm-up period, with \(CV_F\) reaching approximately 17.5 in the evaluated setting. This indicates substantial dispersion in Fisher importance and
is consistent with the Fisher-dominant ranking regime assumed in
Assumption 4, although $CV_F$ alone does not establish the required
rank correspondence.

Figures~\ref{fig:global_accuracy} and
Tables~\ref{tab:pathologicalNU} - \ref{tab:pathological-U}
summarize the experimental evaluation of FedFIbOS across
CIFAR-10 and CIFAR-100 under diverse model and statistical
heterogeneity settings. The learning curves consistently
show that FedFIbOS converges faster during the early
communication rounds and attains higher final test accuracy
than the competing baselines. The performance gains become
more pronounced under severe model heterogeneity and non-IID
data distributions. Table~\ref{tab:modeldist} shows that FedFIbOS consistently
achieves the highest classification accuracy across
both uniform and non-uniform settings, demonstrating
the effectiveness of Fisher-guided submodel selection
under heterogeneous federated learning.
The
quantitative results reported in the tables further validate
these observations, demonstrating that FedFIbOS achieves
consistent improvements over FIARSE, FedBRB, FedRolex, and
HeteroFL across all evaluated experimental settings.

\begin{figure}[t]
\centering
\includegraphics[width=.7\linewidth]{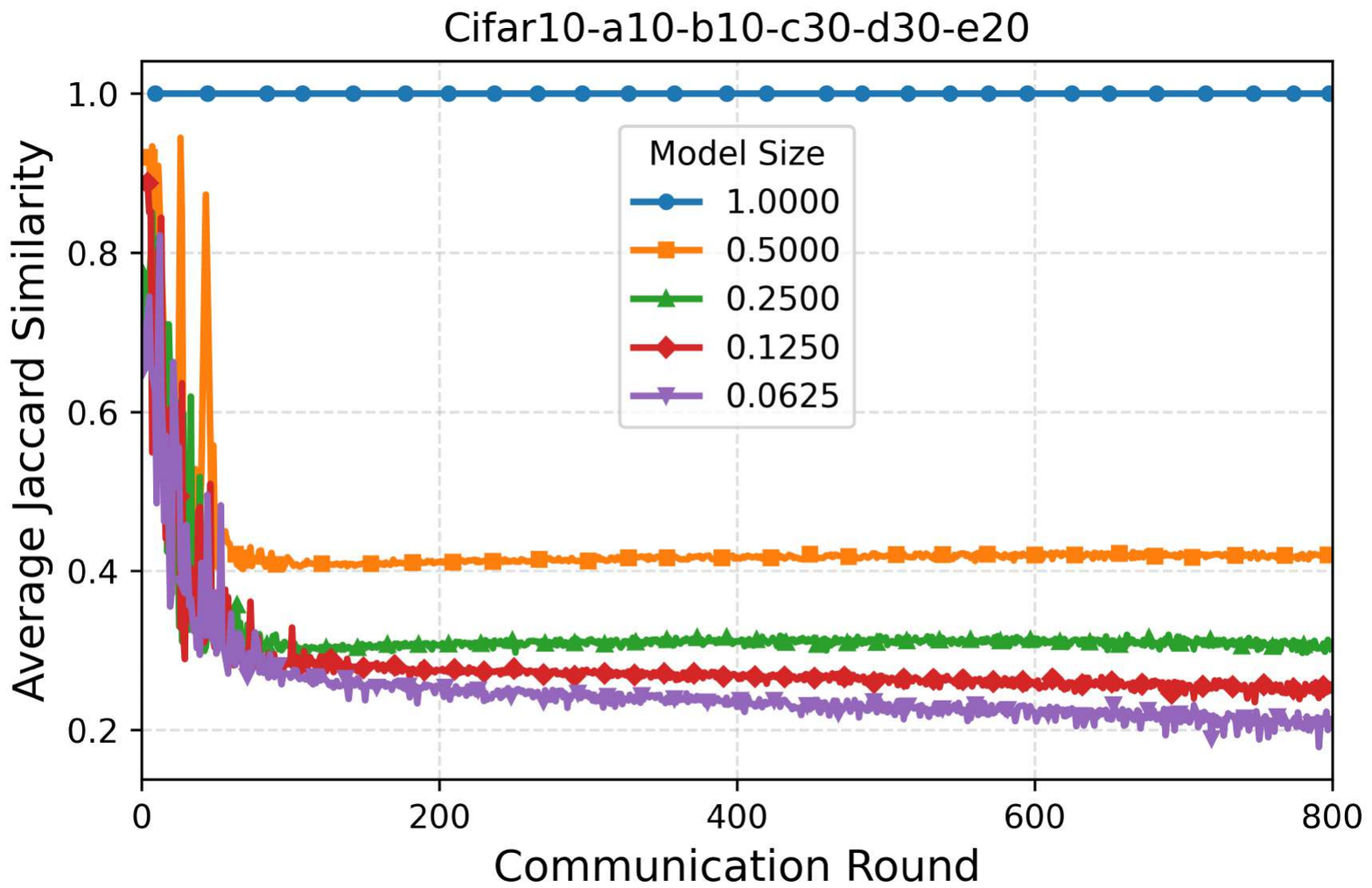}
\caption{Average Jaccard similarity between Fisher-guided
and magnitude-guided parameter selection on CIFAR-10
(non-IID, $\alpha{=}0.3$) for different model sizes.}
\label{fig:jaccard}
\end{figure}

\begin{figure}[t]
\centering
\includegraphics[width=.7\linewidth]
{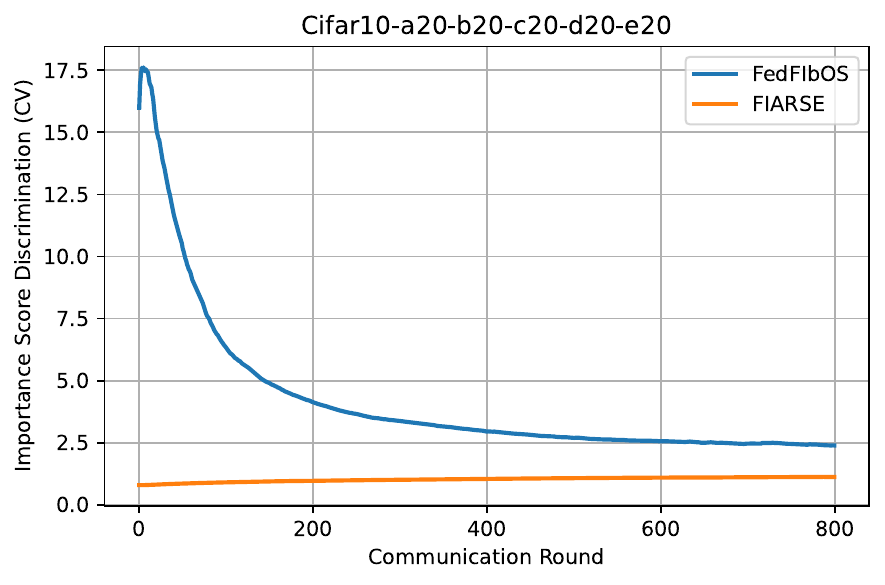}
\caption{Fisher score heterogeneity $\CVF$ (FedFIbOS) versus
magnitude score heterogeneity $\mathrm{CV}_{\mathrm{mag}}$
on CIFAR-10 (non-IID, $\alpha{=}0.3$).}
\label{fig:cvf}
\end{figure}

\begin{table*}[ht]
\centering
\caption{Comparison of classification accuracy (\%) achieved by different federated learning algorithms under non-IID data and non-uniform model distribution a10-b10-c30-d30-e20. The best result is shown in \textbf{bold} and the second-best result is \underline{underlined}.}
\label{tab:pathologicalNU}

\resizebox{.8\textwidth}{!}{
\begin{tabular}{ll|cc|cc|cc|cc|cc}
\toprule

\multirow{2}{*}{\textbf{Dataset}} &
\multirow{2}{*}{\textbf{Heterogeneity}} &
\multicolumn{2}{c|}{\textbf{HeteroFL}} &
\multicolumn{2}{c|}{\textbf{FedRolex}} &
\multicolumn{2}{c|}{\textbf{FedBRB}} &
\multicolumn{2}{c|}{\textbf{FIARSE}} &
\multicolumn{2}{c}{\textbf{FedFIbOS (Ours)}} \\

\cmidrule(lr){3-4}
\cmidrule(lr){5-6}
\cmidrule(lr){7-8}
\cmidrule(lr){9-10}
\cmidrule(lr){11-12}

&
&
\textbf{Local} & \textbf{Global} &
\textbf{Local} & \textbf{Global} &
\textbf{Local} & \textbf{Global} &
\textbf{Local} & \textbf{Global} &
\textbf{Local} & \textbf{Global} \\

\midrule

\multirow{2}{*}{CIFAR-10  } & Low  & 82.26 & 70.36 & 78.15 & 61.12 & 81.34 & 74.32 & \underline{84.89} & \underline{77.85}
 & \textbf{92.10} & \textbf{90.14}\\
 &High & 84.05 & 34.44 & 83.92 &39.01 & 89.9 & 54.48 & 
\underline{91.19} & \underline{59.85} & \textbf{92.80} & \textbf{92.34} \\

\midrule

\multirow{2}{*}{CIFAR-100  } & Low  & \underline{56.83} &  \underline{44.37}& 43.49 & 29.34 & 47.30 & 37.10 & 41.49  & 39.67&\textbf{55.99} & \textbf{ 54.34} \\

 & High & 60.15 & 35.27 & 54.24 & 27.86& 59.25 & 31.74  & \underline{39.69} & \underline{36.09} & \textbf{57.09} & \textbf{ 54.43}\\

\bottomrule
\end{tabular}
}
\end{table*}

\begin{table*}[ht]
\centering
\caption{Comparison of classification accuracy (\%) achieved by different federated learning algorithms under non-IID data and uniform model distribution a20-b20-c20-d20-e20. The best result is shown in \textbf{bold} and the second-best result is \underline{underlined}.}
\label{tab:pathological-U}

\resizebox{.8\textwidth}{!}{
\begin{tabular}{ll|cc|cc|cc|cc|cc}
\toprule

\multirow{2}{*}{\textbf{Dataset}} &
\multirow{2}{*}{\textbf{Heterogeneity}} &
\multicolumn{2}{c|}{\textbf{HeteroFL}} &
\multicolumn{2}{c|}{\textbf{FedRolex}} &
\multicolumn{2}{c|}{\textbf{FedBRB}} &
\multicolumn{2}{c|}{\textbf{FIARSE}} &
\multicolumn{2}{c}{\textbf{FedFIbOS (Ours)}} \\

\cmidrule(lr){3-4}
\cmidrule(lr){5-6}
\cmidrule(lr){7-8}
\cmidrule(lr){9-10}
\cmidrule(lr){11-12}

&
&
\textbf{Local} & \textbf{Global} &
\textbf{Local} & \textbf{Global} &
\textbf{Local} & \textbf{Global} &\textbf{Local} & \textbf{Global} &
\textbf{Local} & \textbf{Global} \\

\midrule

\multirow{2}{*}{CIFAR-10}    & Low  & 88.43 &  79.04 & 86.32 & 72.55 & 81.09 & 78.50 &  \underline{88.80}&\underline{83.95} 
& \textbf{91.80} & \textbf{89.84}\\
 & High & 89.70 & 46.03 & 91.08 & 56.82 & 89.34 & 52.59 
 & \underline{91.75} & \underline{75.76} & \textbf{94.50} & \textbf{92.85} \\

\midrule

\multirow{2}{*}{CIFAR-100}   & Low  & \underline{60.12} & \underline{50.24}& 65.84 & 37.90 & 52.23 & 44.44 & 47.75 & 43.97 & \textbf{56.69} & \textbf{54.09} \\
 & High &\underline{69.82} & \underline{45.13} & 52.68 & 38.08&62.40 & 40.15 & 48.65 &44.37&\textbf{57.75} & \textbf{54.70} \\

\bottomrule
\end{tabular}
}
\end{table*}

\begin{table}[t]
\centering
\footnotesize
\setlength{\tabcolsep}{3pt}
\caption{Comparison of classification accuracy (\%) under uniform (U) and non-uniform (NU) model distributions. CIFAR-10 and CIFAR-100 use Dirichlet $\alpha=0.3$, while AGNews uses $\alpha=1.0$. The best result is shown in \textbf{bold}.}
\label{tab:modeldist}
\resizebox{.4\textwidth}{!}{
\begin{tabular}{ll|cc|cc}
\toprule
\multirow{2}{*}{\textbf{Dist.}} &
\multirow{2}{*}{\textbf{Dataset}} &
\multicolumn{2}{c|}{\textbf{FIARSE}} &
\multicolumn{2}{c}{\textbf{FedFIbOS}\textbf{ (Ours)}} \\
\cmidrule(lr){3-4}
\cmidrule(lr){5-6}
&
&
\shortstack{\textbf{Local}\\\textbf{Accuracy}} &
\shortstack{\textbf{Global}\\\textbf{Accuracy}} &
\shortstack{\textbf{Local}\\\textbf{Accuracy}} &
\shortstack{\textbf{Global}\\\textbf{Accuracy}} \\
\midrule

\multirow{3}{*}{U}
& CIFAR-10  & 82.99& 76.80 & \textbf{89.90} & \textbf{85.29} \\
& CIFAR-100 & 41.25 & 41.15 & \textbf{54.75} & \textbf{52.67} \\
& AGNews   & 92.40 & 68.17 & \textbf{95.15} & \textbf{94.43} \\
\midrule

\multirow{3}{*}{NU}
& CIFAR-10  & 79.70 & 68.05 & \textbf{89.90} & \textbf{84.99} \\
& CIFAR-100 & 37.30 & 35.86 & \textbf{53.10} & \textbf{51.33} \\
& AGNews   & 38.72 & 46.25 & \textbf{95.03} & \textbf{94.13} \\

\bottomrule
\end{tabular}
}
\end{table}

\section{Theoretical Analysis}
\label{sec:theory}

Fisher top-k solves SSP under a Fisher-dominant ranking condition. We use five working
assumptions: Assumptions 1--3 are standard regularity conditions
in heterogeneous FL (Wu et al. 2024), Assumption 4 connects the
Fisher-weighted SSP solution to Fisher top-$k$, and Assumption 5
connects the surrogate to the masking-error coefficient used in
the convergence bound.
\label{sec:assumptions}
\begin{assumption} ($L$-smoothness).
\label{assump:Lsmooth}
For all $n\in[N]$ and $w,v\in\mathbb{R}^d$:
$\|\nabla_w F_n(w\odot M_F^{(n)})-\nabla_v F_n(v\odot
M_F^{(n)})\|_2\leq L\|w-v\|_2$.

\end{assumption}

\begin{assumption} (Bounded variance).
\label{assump:bounded_variance}
$\frac{1}{N}\sum_{n=1}^N\|\nabla F_n(w)
-\nabla F(w)\|_2^2\leq\sigma^2$ for all $w$.
This depends only on $\{P_n\}$, not the selection
criterion.

\end{assumption}
\begin{assumption} (Masking error bound).
\label{assump:masking_error}
For each round $t$, there exists
$\delta_t^{F,2}\in[0,1)$ such that:
\begin{equation}
\left\|
\nabla_{\tilde{x}_t}
F_n\left(
\tilde{x}_t \odot M_F^{(n)}
\right)
\right\|_2^2
\leq
\delta_{F,n,t}^2
\left\|
\tilde{x}_t
\right\|_2^2.
\label{eq:mask_err}
\end{equation}

Let
\begin{equation}
\delta_{F,t}^2
:=
\max_{n\in[N]}
\delta_{F,n,t}^2,
\qquad
\delta_{M,t}^2
:=
\max_{n\in[N]}
\delta_{M,n,t}^2.
\end{equation}

where $\delta_{M,n,t}^2$ is the corresponding client-specific
coefficient under magnitude selection. Thus, $\delta_{F,t}^2$
and $\delta_{M,t}^2$ are the global worst-client masking-error
coefficients used in the convergence bound.

\end{assumption}

\begin{assumption} (Fisher-dominant ranking).
For each client \(n\), round \(t\), and any pair of parameters \(j,k\),
\begin{equation}
F_{t,j}^{(n)}>F_{t,k}^{(n)}
\Longrightarrow
I_{t,j}^{(n)}>I_{t,k}^{(n)},
\qquad
I_{t,j}^{(n)}
=
(F_{t,j}^{(n)})^2\tilde{x}_{t,j}^2.
\end{equation}

Thus, the ordering induced by raw Fisher scores agrees with the
ordering of the Fisher-weighted surrogate scores. This condition
is precisely what makes the raw Fisher top-$k$ rule used by
FedFIbOS coincide with the exact solution of the quadratic SSP.
It is a ranking assumption and is not implied by non-IID data alone.
\end{assumption}

\begin{assumption}[Monotonicity of masking error]
\label{assump:monotonic_masking}
For each client $n$ and round $t$, the masking-error coefficient
$\delta_{S,n,t}^2$ is monotone with respect to the
Fisher-weighted quadratic masking-error surrogate. Specifically,
for any two feasible selections $S_1$ and $S_2$,

\begin{equation}
\epsilon_t^{(n)}(S_1)\leq \epsilon_t^{(n)}(S_2)
\quad\Longrightarrow\quad
\delta_{S_1,n,t}^2\leq \delta_{S_2,n,t}^2.
\end{equation}

This order-preservation condition is sufficient to transfer the
surrogate ordering established by Proposition~1 to the
selection-dependent term in the convergence bound.
\end{assumption}

\begin{remark}[Selection agreement]
\label{lem:equiv}

The two selection rules have identical surrogate values whenever
they select the same parameter set. This requires agreement of
their rankings at the relevant top-$k_n$ boundary, but does not
require their complete rankings to coincide.

\end{remark}
 
\subsection{Magnitude Selection Suboptimal Under Non-IID}

\label{sec:lemma1}

\begin{lemma}[Conditions for Selection Divergence]
Suppose that at round $t+1$ there exists a pair of parameters
$(j,k)$ straddling the top-$k_n$ selection boundary such that
$F_{t+1,j}^{(n)}>F_{t+1,k}^{(n)}$, with $j$ selected and $k$
excluded by the Fisher ranking, and
suppose their magnitudes at the preceding round satisfy
$|\tilde{x}_{t,j}|\approx|\tilde{x}_{t,k}|$.
If the stochastic update has a nonzero probability of producing
\[
|\tilde{x}_{t+1,k}|>|\tilde{x}_{t+1,j}|,
\]
then
\begin{equation}
\label{eq:div}
  \Pr\!\left(
S_M^{(n,t+1)}\neq S_F^{(n,t+1)}
\right)>0.  
\end{equation}

\end{lemma}
\begin{proof} 
Under non-IID data, the client-specific Fisher scores need not be
uniform across parameters. Consider the stated boundary pair $(j,k)$, with
$F_{t+1,j}^{(n)}>F_{t+1,k}^{(n)}$ and approximately equal
magnitudes at round $t$.
After the stochastic local update
$\tilde{x}{t+1,j}=\tilde{x}{t,j}-\eta_l g_{t,j}$, and analogously
for $k$, suppose the update has positive probability of producing
$|\tilde{x}{t+1,k}|>|\tilde{x}{t+1,j}|$. On this event, the
magnitude ranking places $k$ ahead of $j$, whereas the Fisher
ranking retains $j$ ahead of $k$ whenever
$F_{t+1,j}^{(n)}>F_{t+1,k}^{(n)}$. Thus the two rankings differ at
the boundary and consequently their top-$k_n$ selections differ.
Since the magnitude-order reversal occurs with positive
probability by assumption, the selection divergence in Equation \eqref{eq:div}
follows.
\end{proof}
\begin{remark}
Jaccard$(S^F,S^M)\approx 0.22$--$0.42$ empirically
(Figure~\ref{fig:jaccard}) reveals significant differences between $S^F$ and $S^M$ throughout training.
\end{remark}

\subsection{Fisher Optimality and Masking Error Reduction}
\label{sec:thm1}
\label{thm:optimality}
Theorem 1 is stated for the ideal Fisher scores
$F_{t,j}^{(n)}$. In the practical algorithm, these scores are
approximated by the client-wise EMA $\bar{f}_j^{(n,t)}$; the
selection rule in Eq.~(10) therefore implements an empirical-Fisher
approximation to the ideal raw Fisher top-$k$ rule.
\begin{theorem}[Optimality of Fisher Selection]
\label{thm:fisher_optimality}
For fixed client $n$, round $t$, and cardinality
$k_n=\lceil\gamma_n d\rceil$, the solution of SSP (5) is
\begin{equation}
    S^{\star(n,t)}
=
\operatorname{TopK}_{k_n}
\left(
I_{t,j}^{(n)}
\right),
\qquad
I_{t,j}^{(n)}
=
(F_{t,j}^{(n)})^2\tilde{x}_{t,j}^2.
\end{equation}

Under Assumption 4,
\[
S^{\star(n,t)}
=
\operatorname{TopK}_{k_n}
\left(
F_{t,j}^{(n)}
\right)
=
S_F^{(n,t)}.
\]
Hence, under Assumption 4, the raw Fisher top-$k$ rule implemented
by FedFIbOS solves the quadratic SSP. The solution is unique when
the relevant importance scores are distinct.
\end{theorem}

\begin{proof}
Since
\[
\epsilon_t^{(n)}(S)
=
\sum_{j\notin S}I_{t,j}^{(n)}
=
\sum_{j=1}^{d}I_{t,j}^{(n)}
-
\sum_{j\in S}I_{t,j}^{(n)},
\]
minimizing $\epsilon_t^{(n)}(S)$ over all sets of size $k_n$
is equivalent to selecting the $k_n$ largest
$I_{t,j}^{(n)}$, giving
\[
S^{\star(n,t)}
=
\operatorname{TopK}_{k_n}(I_{t,j}^{(n)}).
\]
By Assumption 4, the ordering of $I_{t,j}^{(n)}$ agrees with that
of $F_{t,j}^{(n)}$. Therefore,
\[
\operatorname{TopK}_{k_n}(I_{t,j}^{(n)})
=
\operatorname{TopK}_{k_n}(F_{t,j}^{(n)}),
\]
which is exactly the raw Fisher top-$k$ rule used by FedFIbOS.
Distinct scores imply uniqueness.
\end{proof}
\begin{remark}[Fisher-score dispersion]
The coefficient of variation $CV_F$ measures the relative dispersion
of Fisher importance scores.Figure~4 (Experiments) shows that $CV_F$ becomes substantially larger than
$CV_{\mathrm{mag}}$ after the warm-up period. This demonstrates substantial heterogeneity in Fisher-based parameter importance; however, \(CV_F\) alone does not verify the ranking condition in Assumption 4. We
do not claim that $CV_F$ alone determines the masking-error gap; the
specific magnitude of $CV_F$ is reported and discussed in the Experiments
section.
\end{remark}

\begin{proposition}[Quadratic Masking-Error Improvement]
\label{prop:masking}
Let $S_F^{(n,t)}$ and $S_M^{(n,t)}$ denote the raw Fisher and
magnitude-based selections, respectively. Under Assumption 4,
\begin{equation}
\label{eq:ep_comp}
    \epsilon_t^{(n)}(S_F^{(n,t)})
\leq
\epsilon_t^{(n)}(S_M^{(n,t)}).
\end{equation}

If the two selected sets differ and the relevant importance scores
are distinct, the inequality is strict.

\end{proposition}

\begin{proof}
Under the stated coordinate-activity condition, the EMA for parameter $j$ receives one unbiased Fisher observation whenever client $n$ participates. Conditioning on $R_n(t)=r$ therefore gives the standard $r$-step EMA recursion, yielding \eqref{eq:ema_bias}. Under uniform participation, $R_n(t)\sim\mathrm{Binomial}(t,p)$, which gives \eqref{eq:ema_participant}.
\end{proof}

\subsection{Convergence Guarantee}
\label{sec:thm2}

The convergence bound for masked heterogeneous FL takes the form
\begin{equation}
\label{eq:CG}
\min_{t\in[T]}
\|\nabla F(\tilde{x}_t)\|_2^2
\le
C_1(T)+C_2+
\frac{32N}{T}
\sum_{t=0}^{T-1}
\delta_t^2\|\tilde{x}_t\|_2^2,
\end{equation}
where $C_1(T)\to0$ as $T\to\infty$ and $C_2$ depends on
data heterogeneity (Wu et al. 2024). The selection criterion
affects the bound through the masking-error term.


\begin{theorem}[(Convergence under Fisher-guided selection)]
\label{thm:convergence}
Under
Assumptions 1--5 , the
Fisher-guided masked federated optimization procedure satisfies

\begin{equation}
\min_{t \in [T]}
\left\|
\nabla F(\tilde{x}_t)
\right\|_2^2
\leq
C_1(T) + C_2
+
\frac{32N}{T}
\sum_{t=0}^{T-1}
\delta_{F,t}^{2}
\left\|
\tilde{x}_t
\right\|_2^2.
\label{eq:lr}
\end{equation}
where $C_1(T)\to0$ as $T\to\infty$. Moreover, under Assumption 4, the ideal raw Fisher selection
$S_F^{(n,t)}=\operatorname{TopK}_{k_n}(F_t^{(n)})$ minimizes the
surrogate among all size-$k_n$ selections. The implemented
selection $\operatorname{TopK}_{k_n}(\bar f^{(n,t)})$ is an EMA-based
approximation to this ideal rule; Proposition 2 characterizes the
decay of its estimation bias under the stated unbiasedness and
fixed-target assumptions. By Assumption~5, the ordering of the masking-error surrogate
is preserved by the corresponding masking-error coefficients.

\begin{equation}
\label{eq:convg_bound}
\delta_{F,t}^2
\le
\delta_{M,t}^2.
\end{equation}

Thus, the selection-dependent contribution of the convergence
bound is no larger for FedFIbOS than for magnitude-based selection,
with strict improvement whenever the corresponding masking-error
coefficients differ
\end{theorem}

\begin{proof}
The stated bound follows from the masked heterogeneous federated
optimization analysis under Assumptions~1--3 for the ideal
raw-Fisher selection. Under Assumption~4, Theorem~1 establishes
that $S_F^{(n,t)}$ minimizes $\epsilon_t^{(n)}(S)$ among all
size-$k_n$ selections, and Proposition~1 therefore gives \eqref{eq:ep_comp}.
 Thus, by \eqref{eq:convg_bound}, the selection-dependent term in the convergence
bound is no larger for FedFIbOS than for magnitude-based selection.
\end{proof}

\begin{remark}
The convergence guarantee follows directly,
whereas prior work~\citep{wu2024fiarse} analyzes
convergence without establishing the optimality of the
selection criterion.
\end{remark}

\begin{proposition}[Fisher EMA under Partial Participation]
\label{prop:ema}
Let $R_n(t)$ denote the number of rounds in which client $n$ has
participated before round $t$. If, whenever client \(n\) participates and parameter \(j\) is active in its selected submodel, the resulting Fisher estimate is unbiased, then, for a fixed target Fisher value
$F_j^{(n)}$,
\begin{equation}
\label{eq:ema_bias}
 \mathbb{E}
\left[
\bar f_j^{(n,t)}\mid R_n(t)=r
\right]
=
\alpha^r\bar f_j^{(n,0)}
+
(1-\alpha^r)F_j^{(n)}.
\end{equation}
Thus, the initialization bias decays geometrically with the number
of participating rounds. Under uniform participation probability
$p=A/N$,
\begin{equation}
\label{eq:ema_participant}
    \mathbb{E}\!\left[\alpha^{R_n(t)}\right]
=
(1-p+p\alpha)^t.
\end{equation}

\end{proposition}

\begin{proof}
The EMA of client $n$ is updated only when the client participates.
Conditioning on $R_n(t)=r$ therefore gives the standard $r$-step
EMA recursion, yielding \ref{eq:ema_bias}. Under uniform participation,
$R_n(t)\sim\operatorname{Binomial}(t,p)$, which gives \ref{eq:ema_participant}.
\end{proof}

\section{Conclusion}

We presented FedFIbOS, a Fisher Information-guided submodel
extraction framework for model-heterogeneous federated
learning. By leveraging client-specific Fisher information,
FedFIbOS constructs submodels that better preserve locally
important parameters under heterogeneous data distributions.
Theoretical analysis characterizes the Fisher-weighted quadratic
submodel-selection objective and establishes that the raw Fisher
top-$k$ rule used by FedFIbOS is optimal under the stated
Fisher-dominant ranking condition. Under the corresponding
monotonicity assumption, this yields a no-larger
selection-dependent term in the convergence bound. The proposed method
incurs no additional computational overhead, as Fisher
information is estimated directly from the squared gradients
computed during local optimization and maintained using a
server-side exponential moving average. Extensive experiments
on CIFAR-10, CIFAR-100, and AGNews demonstrated
consistent improvements over state-of-the-art heterogeneous
federated learning methods, including FIARSE, FedBRB,
FedRolex, and HeteroFL, particularly under data and
resource heterogeneity.


 \bibliography{aaai2027}


\end{document}